\documentclass[runningheads]{llncs}
\usepackage{times}
\usepackage{soul}
\usepackage{url}
\usepackage[hidelinks]{hyperref}
\usepackage[utf8]{inputenc}
\usepackage[small]{caption}
\usepackage{graphicx}
\usepackage{floatrow}
\usepackage{booktabs}
\usepackage{array}
\newcolumntype{L}[1]{>{\raggedright\arraybackslash}p{#1}}
\newcolumntype{C}[1]{>{\centering\arraybackslash}p{#1}}
\newcolumntype{R}[1]{>{\raggedleft\arraybackslash}p{#1}}

\usepackage{amsmath,amsfonts,bm}

\def\eqref#1{equation~\ref{#1}}

\def\1{\bm{1}}

\DeclareMathAlphabet{\mathsfit}{\encodingdefault}{\sfdefault}{m}{sl}
\SetMathAlphabet{\mathsfit}{bold}{\encodingdefault}{\sfdefault}{bx}{n}

\newcommand{\R}{\mathbb{R}}

\usepackage{url}

\usepackage[T1]{fontenc}    
\usepackage{amsfonts,amssymb,dsfont}
\usepackage{graphicx}
\usepackage{subfigure}
\usepackage{booktabs}
\usepackage{float}
\usepackage{wrapfig} 
\usepackage{cite}

\usepackage{color}

\usepackage{todonotes}

\def\SAN{{SAN}}
\def\neu{\mathcal{N}}

\def\lin{\mathrm{span}}
\def\pooling{\mathrm{Pool}}
\def\N{\mathbb{N}}

\def\relu{\mathrm{ReLU}}

\DeclareMathOperator*{\supp}{supp}

\usepackage{color}
\usepackage{todonotes}

\def\relu{\mathrm{ReLU}}

\begin{document}
\title{Set Aggregation Network as a Trainable Pooling Layer}

\author{\L{}ukasz Maziarka\inst{1,2} \and
Marek \'Smieja\inst{1} \and
Aleksandra Nowak\inst{1} \and
Jacek Tabor\inst{1} \and
\L{}ukasz Struski\inst{1} \and
Przemys\l{}aw Spurek\inst{1}}

\authorrunning{L. Maziarka et al.}

\institute{Faculty of Mathematics and Computer Science\\
Jagiellonian University\\
\L{}ojasiewicza 6, 30-348, Krak\'ow, Poland\\
\email{marek.smieja@ii.uj.edu.pl} \and
Ardigen \\ Podole 76, 30-394, Krak\'ow, Poland}

\maketitle              

\begin{abstract}
Global pooling, such as max- or sum-pooling, is one of the key ingredients in deep neural networks used for processing images, texts, graphs and other types of structured data.
Based on the recent DeepSets architecture proposed by Zaheer et al. (NIPS 2017), we introduce a Set Aggregation Network (\SAN{}) as an alternative global pooling layer. In contrast to typical pooling operators, \SAN{} allows to embed a given set of features to a vector representation of arbitrary size. We show that by adjusting the size of embedding, \SAN{} is capable of preserving the whole information from the input. In experiments, we demonstrate that replacing global pooling layer by \SAN{} leads to the improvement of classification accuracy. Moreover, it is less prone to overfitting and can be used as a regularizer.

\keywords{Global pooling \and Structured data \and Representation learning \and Convolutional neural networks \and Set processing \and Image processing.}
\end{abstract}
\section{Introduction}

Deep neural networks are one of the most powerful machine learning tools for processing structured data such as images, texts or graphs \cite{ronneberger2015u, iizuka2017globally}. While convolutional or recurrent neural networks allow to extract a set of meaningful features, it is not straightforward how to vectorize their output and pass it to the fully connected layers. 

One typical approach to this problem relies on flattening the given tensor. However, the flattened vector may contain a lot of redundant information, which in turn may lead to overfitting. Moreover, flattening cannot be followed by a dense layer (e.g. in classification models), when the input has varied size \cite{karpathy2015deep, ciresan2011flexible}. This situation often appears in graphs and texts classification, but is also common in learning from images of different resolutions \cite{brin1998anatomy,frasconi1998general, he2015spatial}. 

In order to produce a fixed-length vector, a maximum or sum function may be applied to the learned data representations. This operation is commonly known as the global pooling. In image recognition the data is frequently aggregated by computing the average value over the channels of the feature map tensor obtained by the backbone network. Such vector is then passed to the predictor head. This is the case in numerous large scale networks such as ResNet\cite{he2016resnet}, DenseNet\cite{Huang2017DenselyCC} or, more recent, Amoeba-Net\cite{real2018amoeba}. In Graph Neural Networks the representation of a given node is usually computed by recursively mean-pooling the representations of its neighbours\cite{kipf2016semi, xu2018how}. Despite its wide applicability, the global pooling layer is not able to fully exploit the information from the input data, because it does not contain trainable parameters. Moreover, the global pooling cannot adjust the dimensionality of the representation, because the size of its result is solely determined by the number of input channels.

An additional requirement often imposed on the aggregation function is the invariance to the permutation of the input. This constraint arises as a consequence of set processing, and is present, for instance, in 3D point cloud recognition \cite{Qi2017PointNetDL} and graph analysis \cite{xu2018how}. The issue of efficiently learning a representation of permutation invariant inputs  was recently studied by Zaheer et al. \cite{zaheer2017deep}, who proposed a unified methodology for the processing of sets by neural networks. Their model, called DeepSets, is based on the idea of summing the representations of the sets elements. This concept was also further developed by \cite{murphy2018janossy} who define the pooling function as the average of permutation-sensitive maps of all reorderings of the sequence of the set members. Those permutation-sensitive functions may be modeled by recurrent networks and allow for learning representations that utilize high order dependencies between the set points already at the aggregation time. Neural networks capable of processing sets where also analyzed by \cite{sannai2019universal}, who prove that the studied permutation invariant/equivariant deep models are universal approximators of permutation invariant/equivariant functions. 

In this paper, we propose Set Aggregation Network (\SAN{}), an alternative to the global pooling operation, which guarantees no information loss. For this purpose, we adapt the DeepSets architecture  to embed a set of features retrieved from structured data into a vector of fixed length. Contrary to pooling operation, the parameters of \SAN{} are trainable and we can adjust the size of its representation. In addition to Zaheer et al. \cite{zaheer2017deep}, we prove that for a sufficiently large latent space, \SAN{} learns a unique representation of every input set, which justifies this approach from a theoretical perspective  (Theorem~\ref{thm:main}). 

Our experimental results confirm that replacing global pooling by \SAN{} leads to the improvement of classification accuracy of convolutional neural networks used in classification (Section \ref{sec:exp}). Moreover, \SAN{} is less prone to overfitting, which allows to use it as a regularizer. The experiments were carried out on a small network architecture as well as on the large-scale ResNet model.

\begin{figure*}[t] 
\centering
    \includegraphics[width=0.9\textwidth]{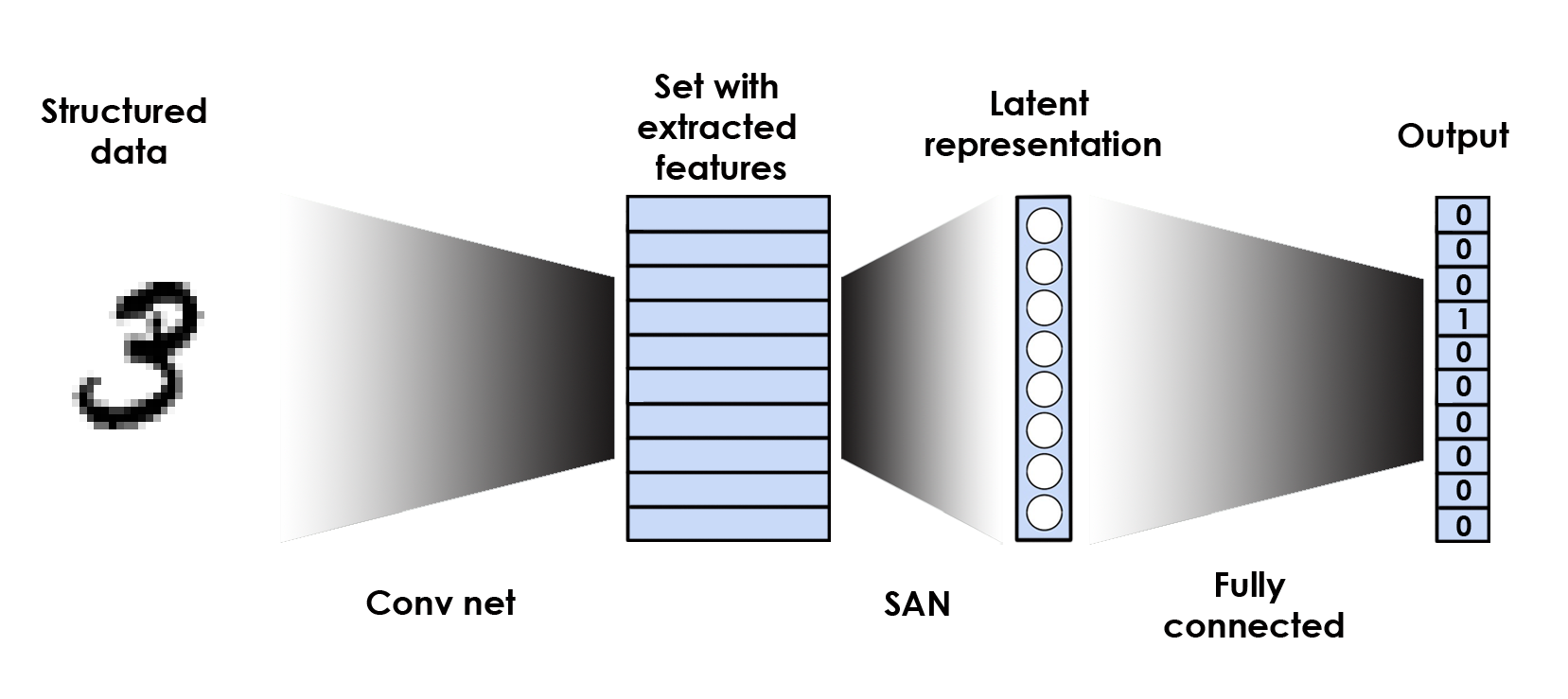}
  \caption{SAN is an intermediate network which is responsible for learning a vector representation using a set of features extracted from of structured data. \label{fig:SAN} } 
\end{figure*}

\section{Set Aggregation Network}

Suppose that we want to process structured data $X = (x_i)_i \subset \R^D$ by the use of a neural network. Some practical examples of $X$ may include a sequence of word embeddings, image represented as a set of pixels or a graph structure. In this paper, we study one of the typical architectures used for processing such data. It consists of two networks combined with an intermediate pooling layer:
\begin{equation}
X = (x_i)_i \stackrel{\Psi}{\to} (\Psi x_i)_i \stackrel{\pooling}{\to}\pooling \{ \Psi(x_i): i \} \stackrel{\Phi}{\to} \R^N.
\end{equation}

The first network $\Psi: \R^{aD} \to \R^K$, where $a \in \N$, is responsible for extracting meaningful features of structured data. In the case of images it can be a convolutional network, while for a sequential data, such as texts, it may be a recurrent neural network. This network transforms elements of $X$ sequentially and produces a set (or sequence) of $K$-dimensional vectors. Depending on the size of $X$ (length of a sentence or image resolution), the number of vectors returned by $\Psi$ may vary. To make the response of $\Psi$ equally-sized, a global pooling is applied, which returns a single $K$-dimensional vector. A pooling layer, commonly implemented as a sum or maximum operation over the set of $K$-dimensional vectors, gives a vector representation of structured object. Finally, a network $\Phi: \R^K \to \R^N$ maps the resulting representation to the final output.

The basic problem with the above pipeline lies in the pooling layer. Global pooling ``squashes'' a set of $K$-dimensional vectors into a single output with $K$ attributes. A single $K$-dimensional output vector may be insufficient to preserve the whole information contained in the input set (especially for large sets and small $K$), which makes a pooling operation the main bottleneck of the above architecture. In this paper, we would like to address this problem. We focus on defining more suitable aggregation network, which will be able to learn a sufficiently rich latent representation of structured data. 

To replace a pooling layer, we extend DeepSets architecture introduced in \cite{zaheer2017deep} to the case of structured data. In consequence, we define an aggregation network, which embeds a set of extracted features to a fixed-length representation. First, we recall a basic idea behind pioneering work of Zaheer et al. and explain its use as an alternative to the classical pooling layer. In the next section, we prove that this framework is able to preserve the whole information contained in the set structure.

Let $f: \R^D \supset X  \to y \in Y$ be a function, which maps sets $X =(x_i)_i$ to some target values $y \in Y$. Since $f$ deals with sets, then the response of $f$ should be invariant to the ordering of set elements. Zaheer et al.~\cite{zaheer2017deep} showed that to realize this requirement $f$ has to be decomposed in the form:
\begin{equation}
f(X) = \rho(\sum_{i} \tau(x_i)),
\end{equation}
for suitable transformations $\rho,\tau$. In the case of neural networks, $f$ can be implemented by constructing two networks $\tau$ and $\rho$. The first network processes elements of a given set $X$ sequentially. Next, the response of $\tau$ is summarized over the whole elements of $X$ and a single vector is returned. Finally, a network $\rho$ maps the resulting representation to the final output.

\begin{figure*}[t] 
\centering
    \includegraphics[width=0.3\textwidth]{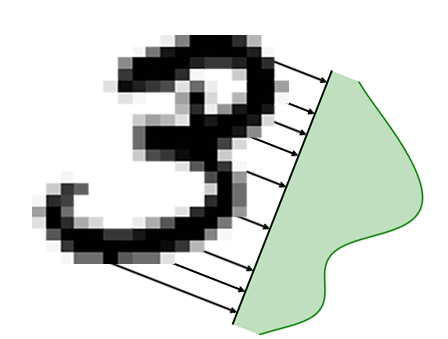}
    \includegraphics[width=0.5\textwidth]{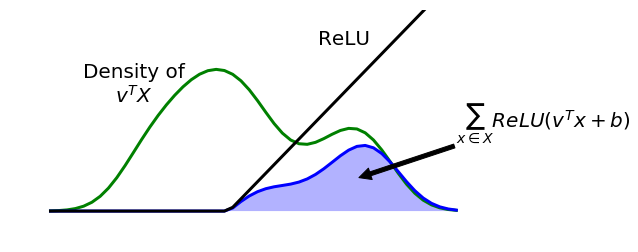}
  \caption{The idea of our approach is to aggregate information from projections of a set onto several one-dimensional subspaces (left). Next non-linear activation function is applied to every set element and the results are aggregated (right). \label{fig:tomography} } 
\end{figure*}

One can directly adapt the above architecture to the pipeline considered in this paper. Namely, instead of taking the maximum or sum pooling over the set of extracted features, we define a separate neural network $\tau$ to compute the summarized activation for all set elements (the role of $\rho$ is played by a network $\Phi$ in our framework). We refer to this network as set aggregation network (\SAN), see Figure \ref{fig:SAN}. If $\tau$ contains $M$ output neurons, then we get $M$-dimensional vector representation of the structured data. In contrast to pooling operation, which always returns $K$-dimensional vector ($K$ is a dimension of input feature vectors), the size of representation produced by \SAN{} may be adjusted to a given problem. Moreover, their parameters are trainable and thus we allow for learning the most convenient representation of structured data. Although \SAN{} is designed to process permutation invariant structures (sets), one may add special attributes to indicate the ordering of extracted features. One way is to use the normalization of the elements index or its trigonometric transformation \cite[Section 3.5]{vaswani2017attention}.

The following remark shows that max-pooling is a special case of \SAN{}.
\begin{remark}
Observe that max-pooling is a special case of \SAN{}. Clearly, for non-negative scalar data $X = (X_i) \subset \R$ and function $\tau_p(x) = x^p$, we have:
\begin{equation}
\tau^{-1}(\sum_i \tau(x_i)) \to \max_i(x_i) \text{ , as } p \to \infty.
\end{equation}
To obtain a maximum, we use $\tau$ as the activity function in aggregative neuron, which is followed by a layer with its inverse. By extending this scheme, we can get a maximum value for every coordinate. Additionally, to deal with negative numbers, we first take the exponent followed by logarithm after the aggregation. 
\end{remark}

\section{Theoretical Analysis}

Although Zaheer et al. theoretically derived the form of $f$ as the only permutation invariant transformation operating on sets, there is no guarantees that this network is capable of learning a unique representation for every set. In this section we address this question and show that if $\tau$ is a universal approximator, then $\sum_{x \in X} \tau(x)$ gives a unique embedding of every set $X$ in a vector space.

Before a formal proof, we first give an illustrative argument for this fact. A typical approach used in computer tomography applies Radon transform \cite{radon1986determination,van2004short} to reconstruct a function (in practice the 2D or 3D image) from the knowledge of its integration over all one-dimensional lines. A similar statement is given by the  Cramer-Wold Theorem \cite{cramer1936some}, which says that for every two distinct measures one can find their one-dimensional projection which discriminates them. This implies that without loss of information we can process the set $X \subset \R^K$ through its all one-dimensional projections $v^T X \subset \R$, where $v \in \R^K$.

In consequence, we reduce the question of representing a multidimensional set to the characterization of one-dimensional sets. Next, one can easily see that the one-dimensional set $S \subset \R$ can be retrieved from the knowledge of aggregated ReLU on its translations: $b \to \sum_i \relu(s_i+b)$, see Figure \ref{fig:tomography}.
Summarizing the above reasoning, we obtain that the full knowledge of a set $X \subset \R^K$ is given by the scalar function
\begin{equation}
\R^K \times \R \ni (v,b) \to \sum_i \relu(v^Tx_i+b). 
\end{equation}
Now, given $M$ vectors $v_i \in \R^K$ and biases $b_i \in \R$, we obtain the fixed-size representation of the set $X \subset \R^K$ as a point in $\R^M$ given by
\begin{equation}
[\sum_i \relu(v_1^Tx_i+b_1),\ldots,\sum_i \relu(v_M^Tx_i+b_M)] \in \R^M.
\end{equation}
The above transformation directly coincides with a single layer \SAN{} parametrized by ReLU function. Thus for a sufficiently large number of neurons, \SAN{} allows to uniquely identify every input set.

Now, we show formally that the above reasoning can be extended to a wide range of activity functions. For this purpose, we will use the UAP (universal approximation property). We say that a family of neurons $\neu$ has UAP if for every compact set $K \subset \R^D$ and a continuous function $f:K \to \R$ the function $f$ can be arbitrarily close approximated with respect to supremum norm  by $\lin (\neu)$ (linear combinations of elements of $\neu$). We show that if a given family of neurons satisfies UAP, then the corresponding \SAN{} allows to distinguish any two  sets:

\begin{theorem} \label{thm:main}
Let $X,Y$ be two sets in $\R^D$. Let $\neu$ be a family of functions having UAP.

If
\begin{equation} \label{eq:1}
\tau(X)=\tau(Y) \mbox{ , for every }\tau \in \neu,
\end{equation}
then $X = Y$.
\end{theorem}
\begin{proof}
Let $\mu$ and $\nu$ be two measures representing sets $X$ and $Y$, respectively, i.e. $\mu = \1_{X}$ and $\nu = \1_Y$. We show that if $\tau(X) = \tau(Y)$ then $\mu$ and $\nu$ are equal.

Let $R>1$ be such that $X \cup Y \subset B(0,R-1)$, where $B(a,r)$ denotes the closed ball centered at $a$ and with radius $r$. To prove that measures $\mu,\nu$ are equal it is sufficient to prove that they coincide on each ball $B(a,r)$ with arbitrary $a \in B(0,R-1)$ and radius $r<1$.

Let $\phi_n$ be defined by
\begin{equation}
\phi_n(x)=1-n \cdot d(x,B(a,r)) \mbox{ for }x \in \R^D,
\end{equation}
where $d(x,U)$ denotes the distance of point $x$ from the set $U$.
Observe that $\phi_n$ is a continuous function which is one on $B(a,r)$ an and zero on $\R^D \setminus B(a,r+1/n)$, and therefore $\phi_n$ is a uniformly bounded sequence of functions which converges pointwise to the characteristic funtion $\1_{B(a,r)}$ of the set $B(a,r)$.

By the UAP property we choose $\psi_n \in \lin(\neu)$ such that 
\begin{equation}
\supp_{x \in B(0,R)} |\phi_n(x)-\psi_n(x)| \leq 1/n.
\end{equation}
Thus $\psi_n$ restricted to $B(0,R)$ is also a uniformly bounded sequence of functions which converges pointwise to $\1_{B(a,r)}$. 
Since $\psi_n \in \neu$, by (\ref{eq:1}) we get
\begin{equation}
\sum_{x \in X} \mu(x) \psi_n(x) =\sum_{y \in Y} \nu(y) \psi_n(y).
\end{equation}
Now by the Lebesgue dominated convergence theorem
we trivially get
\begin{equation}
\begin{array}{l}
\displaystyle{\sum_{x \in X} \mu(x) \psi_n(x)=\int_{B(0,R)} \psi_n(x) d\mu(x) \to \mu(B(a,r)),} \\[1ex]
\displaystyle{\sum_{y \in Y} \nu(y) \psi_n(y)=\int_{B(0,R)} \psi_n(x) d\nu(x) \to \nu(B(a,r)),}
\end{array}
\end{equation}
which completes proof.
\end{proof}

\section{Experiments} \label{sec:exp}

We apply \SAN{} to typical architectures used for processing images. Our primary goal is to compare \SAN{} with global pooling in various settings. First, we focus on classifying images of the same sizes using a small convolutional neural network. Next, we extend this experiment to the case of images with varied resolutions. Finally, we consider a large scale ResNet architecture and show that replacing a global pooling layer by \SAN{} leads to the improvement of classification performance. Our implementation is available at github\footnote{\url{https://github.com/gmum/set-aggregation}.}.

\begin{table*}[t]
\caption{Architecture summary of a small neural network ($N$ is the size of the input to the layer, and $D$ is the number of output neurons from the \SAN{} layer).}
\label{tab:architecture_CV}

\centering
\begin{tabular}{ccc | ccc | ccc}
\toprule
\multicolumn{3}{c |}{Flatten} & \multicolumn{3}{c |}{Max/Avg-pooling} & \multicolumn{3}{c}{SAN} \\
Type & Kernel & Outputs &         Type & Kernel & Outputs &         Type & Kernel & Outputs \\
\midrule
     Conv 2d &    3x3 &      32 &      Conv 2d &    3x3 &      32 &      Conv 2d &    3x3 &      32 \\
 Max pooling &    2x2 &         &  Max pooling &    2x2 &         &  Max pooling &    2x2 &         \\
     Conv 2d &    3x3 &      64 &      Conv 2d &    3x3 &      64 &      Conv 2d &    3x3 &      64 \\
 Max pooling &    2x2 &         &  Max pooling &    2x2 &         &  Max pooling &    2x2 &         \\
     Conv 2d &    3x3 &      64 &      Conv 2d &    3x3 &      64 &      Conv 2d &    3x3 &      64 \\
     Flatten &        &         &  Max/Avg pooling &    NxN &         &          SAN &        &     $D$ \\
       Dense &        &     128 &        Dense &        &     $D$ &        Dense &        &     10 \\
       Dense &        &      10 &        Dense &        &      10 &              &        &         \\
\bottomrule
\end{tabular}

\end{table*}

\subsection{Small Classification Network}

We considered a classification task on CIFAR-10 \cite{krizhevsky2009learning}, which consists of 60 000 color images of the size 32x32 and the Fashion-MNIST, composed of 70 000 gray-scale pictures of size 28x28 \cite{xiao2017online}. We used small neural networks with 3 convolutional layers with ReLU activity function and a max-pooling between them, see Table \ref{tab:architecture_CV} for details.   
 
To aggregate resulted tensor we considered the following variants: 
\begin{itemize}
\item {\bf flatten}: We flattened a tensor to preserve the whole information from the previous network. 
\item {\bf conv-1x1}: We applied 1x1 convolutions with one channel and flattened the output. This reduces the number of parameters in subsequent dense layer compared to the classical flatten approach.
\item {\bf max-pooling}: We applied max pooling along spatial dimensions (width and height of a tensor) to reduce the dimensionality. In consequence, we obtained a vector of the size equal the depth of the tensor.  
\item {\bf avg-pooling}: We considered a similar procedure as above, but instead of max pooling we used average pooling. 
\item {\bf SAN}: We used a single layer \SAN{} as an alternative aggregation. The resulting tensor was treated as a set of vectors with sizes equal to the depth of the tensor. Moreover, the (normalized) indices were added to every vector to preserve the information about local structure. 
\end{itemize}

\SAN{} allows to select the number of output neurons. For the experiment, we considered the following numbers of output neurons: $\{128,256,512,1024,2048,4096\}$. To use the same number of parameters for global pooling and conv1x1 approaches we followed them by a dense network with identical number of neurons to \SAN{}. In the case of flatten, we obtained a comparable number of parameters to the other networks trained on the size 4 096. In each case we used additional two dense layers, except for the \SAN{} model, where only one dense layer was used. All models were trained using adam optimizer with a learning rate $10^{-3}$ and batch size $256$. 
We used 5 000 images for the validation set, 5 000 images for the test set for both CIFAR-10 and Fashion-MNIST. We trained every model on the remaining images.

\begin{table*}[t]
\caption {Classification accuracy on a small network (images with the same resolutions) for different number of parameters used in aggregation layer.}
\label{tab:regularization_acc}
\centering
\footnotesize
\begin{tabular}{l|ccccc|ccccc}
\toprule
    &  \multicolumn{5}{c|}{Fashion MNIST} &  \multicolumn{5}{c}{CIFAR 10} \\ \cmidrule{2-11}
  size    &  flatten & avg-pool &  max-pool & conv-1x1 &  \SAN{} &  flatten & avg-pool &  max-pool & conv-1x1 &  \SAN{}\\
\midrule
128 & -- & 0.852 & 0.901 & {\bf 0.916} & 0.912  & -- & 0.624 & 0.690 & 0.731 & {\bf 0.738} \\ 
256 & -- & 0.877 & 0.908 & {\bf 0.916} & 0.911 &  -- & 0.649 & 0.697 & 0.738 & {\bf 0.739} \\
512 & -- & 0.879 & 0.906 & {\bf 0.918} & 0.910  & -- & 0.671 & 0.701 & 0.722 & {\bf 0.730} \\
1024 & -- & 0.888 & 0.914 & 0.912 & {\bf 0.915}  & -- & 0.659 & 0.683 & 0.722 & {\bf 0.756} \\
2048 & -- & 0.889 & 0.912 & 0.914 & {\bf 0.919}   &  -- & 0.697 & 0.686 & 0.707 & {\bf 0.733} \\
4096 & {\bf0.919} & 0.912 & 0.900 & 0.914 & 0.912  & 0.720 & 0.738 & 0.708 & 0.709 & {\bf 0.762} \\
\bottomrule
\end{tabular}
\end{table*}

\begin{figure*}[h] 
\centering
    \includegraphics[width=0.98\textwidth]{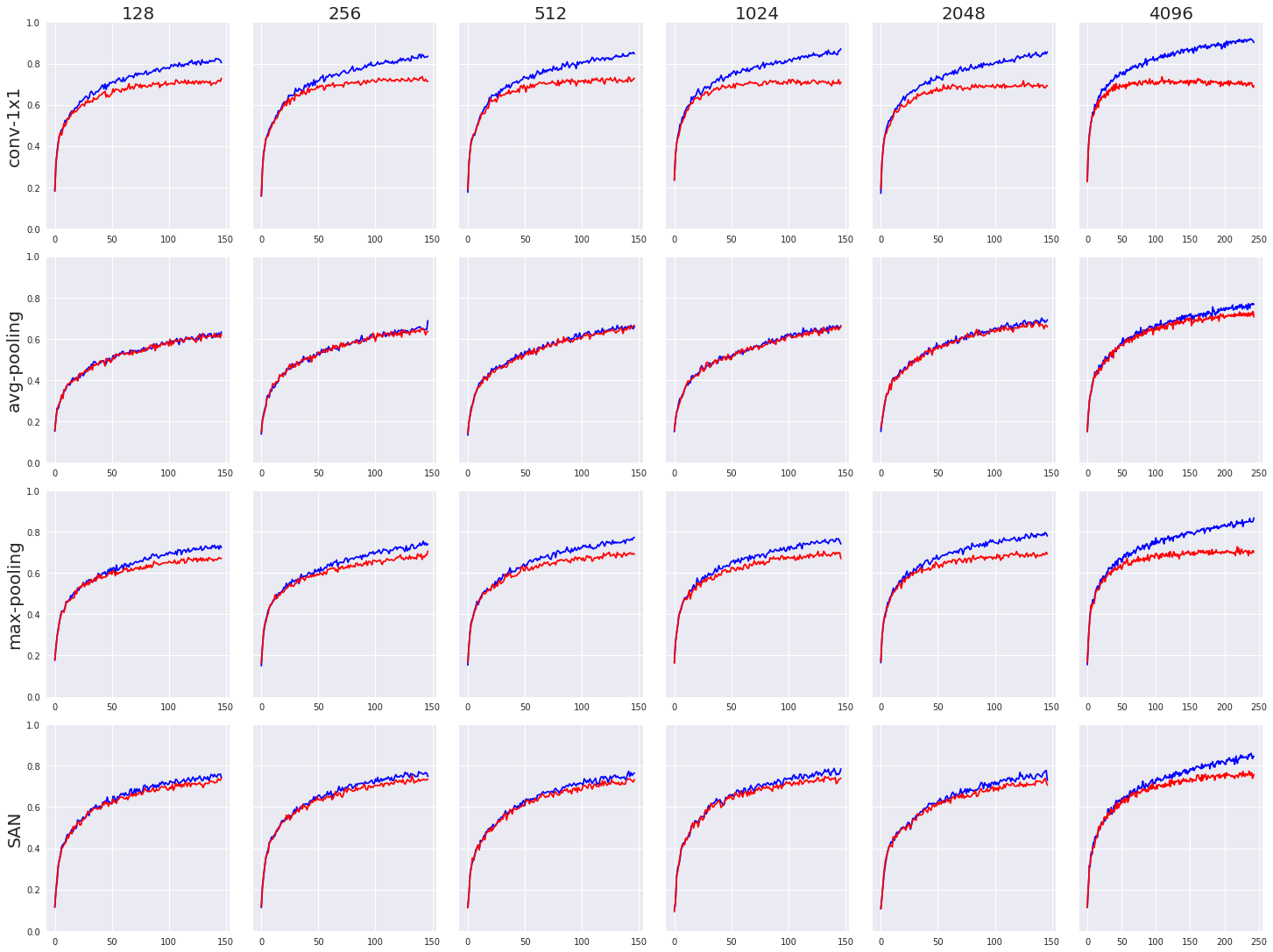}
  \caption{Train (blue) and test accuracy (red) on CIFAR-10 (images with the same resolutions) for different number of parameters used in aggregation layer. \label{fig:regularization}}
\end{figure*}

\begin{figure*}[h] 
\centering
    \includegraphics[width=0.98\textwidth]{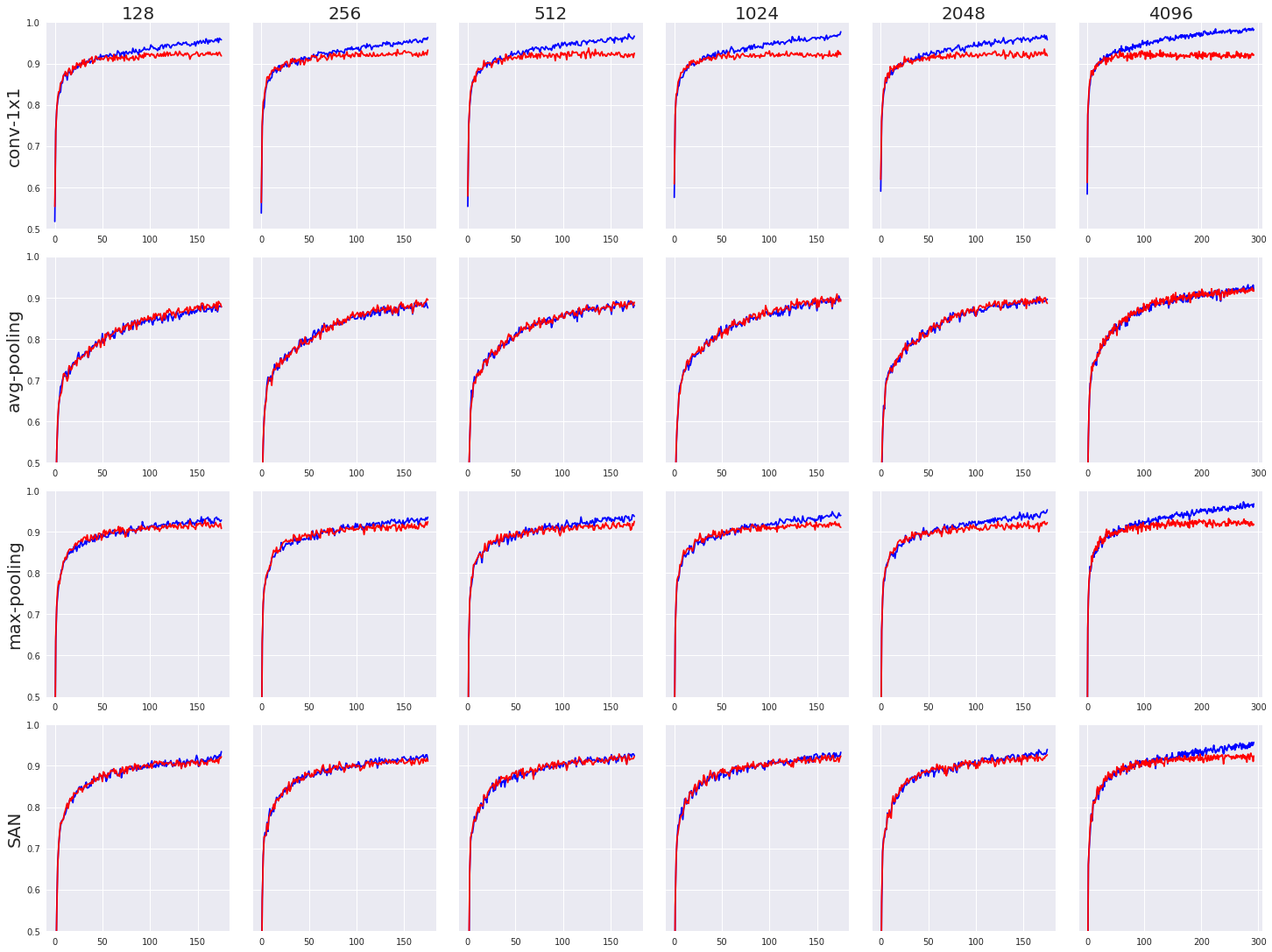}
  \caption{Train (blue) and test accuracy (red) on Fashion-MNIST (images with the same resolutions) for different number of parameters used in aggregation layer. \label{fig:regularization_fashion}}
\end{figure*}

It is evident from  Table \ref{tab:regularization_acc} that \SAN{} outperformed all reported operations on CIFAR-10. In addition, it gave higher accuracy than flatten when both approaches have a comparable number of parameters (last row). We verified that lower results of flatten were caused by its slight overfitting to the training data. Adding dropout to flatten makes both approaches comparable. In the case of significantly simpler Fashion-MNIST dataset, the differences between all methods are smaller. \SAN{} achieved an identical accuracy to flatten for the size $2048$. Note however, that the flatten approach uses twice as much parameters as \SAN{}. This demonstrates that by the use of \SAN{} the network can be simplified without any loss in accuracy.

To get more insight to the results, we present learning curves for CIFAR-10 and Fashion-MNIST in Figure \ref{fig:regularization} and Figure \ref{fig:regularization_fashion}, respectively. It is evident that max-pooling and conv-1x1 highly overfitted to training data, especially on the more demanding CIFAR-10 dataset. Although avg-pooling presented comparable accuracy on train and test sets, its overall performance was quite low, and matched that of the other methods only for high number of the parameters. In contrast, \SAN{} provided high accuracy and did not suffer from high overfitting to training data. In addition, for the Fashion-MNIST experiment, \SAN{} converged to high performance (over 90\%) much faster than the corresponding avg-pooling approach.

This experiment  partially confirms our theoretical result that for a sufficient number of neurons, \SAN{} is able to preserve the whole information contained in the input. On the other hand, it shows that \SAN{} can work as a regularizer, which prevents the model from overfitting. 

\subsection{Classifying Images with Varied Sizes}\label{sec:im}

\begin{table*}[t]
\caption {Classification accuracy for images with varied resolutions.}
\label{tab:resIm}
\centering
\begin{tabular}{lc|ccc|ccc}
\toprule
   &       &  \multicolumn{3}{c|}{Trained on all resolutions} &  \multicolumn{3}{c}{Trained only on original resolution} \\ \cmidrule{3-8}
Dataset   &   Image size    &  max-pool &  avg-pool &     \SAN{} &  max-pool &  avg-pool &  \SAN{} \\
\midrule
Fashion MNIST & 14x14 &       0.8788 & 0.8753 & \textbf{0.8810} &       0.2519 & 0.270 &       \textbf{0.2884}  \\
         & 22x22 &       0.8969 & 0.9002  &  \textbf{0.9064} &       0.7380 & 0.801 &         \textbf{0.8247} \\
         & 28x28 &       0.9023 & 0.9078 &  \textbf{0.9111} &       0.9062 & 0.904 &         \textbf{0.9150}  \\
         & 42x42 &       0.9020 & \textbf{0.9041} &  0.9033 &       0.5548 & 0.6511 &         \textbf{0.6893}  \\
         & 56x56 &       0.8913 & 0.8960 &  \textbf{0.8966} &       0.3274 & 0.3809 &         \textbf{0.4515}  \\
\midrule         
CIFAR-10 & 16x16 &       0.5593 &     0.5820 & \textbf{0.6305} &       0.3251 & 0.2714 &         \textbf{ 0.3808} \\
         & 24x24 &       0.6450 &  0.6935 &  \textbf{0.7317} &       0.6409 & 0.6130 &         \textbf{0.6956} \\
         & 32x32 &        0.6729 & 0.7018 &  \textbf{ 0.7565} &       0.7131 &  \textbf{0.7637} &  0.7534 \\
         & 40x40 &        0.6739 &  0.6914 &  \textbf{0.7430} &       0.6512 & 0.6780 &         \textbf{0.7234} \\
         & 48x48 &       0.6770 &    0.6625 &  \textbf{0.7626} &       0.5325 &    0.5366 &         \textbf{0.6264}\\
\bottomrule
\end{tabular}
\end{table*}
Most classification models assume that input images are of the same size. If this is not the case, we are forced to scale images at preprocessing stage or use pooling operation as an intermediate layer to apply fully connected layers afterwards. In this experiment, we compared \SAN{} with max-pooling and avg-pooling in classifying images of different sizes. We used analogical architecture as in previous section. Note that we were unable to use flatten or conv-1x1, because the output from convolutional network had different sizes.

We again considered Fashion-MNIST and CIFAR-10 datasets. To create examples with different sizes we used bicubic interpolation on randomly selected images\footnote{For CIFAR-10, original images of size $32 \times 32$ were scaled to $16 \times 16, 24 \times 24, 32 \times 32, 40 \times 40, 48 \times 48$. For Fashion-MNIST, images of size $28 \times 28$ were scaled to $14 \times 14, 22 \times 22, 42 \times 42, 56~\times~56$.}. We examined two cases. In the first one, the network was trained only on images with original resolution, but tested on images with different resolutions. In the second case, scaled images were used both in training and testing. 

The results presented in Table \ref{tab:resIm} show that {\bf \SAN} produced more accurate results than both global pooling approaches for almost every image resolution. Observe that the results are worse when only images with $32 \times 32$ size were used in train set. It can be explained by the fact that convolutional filters were not trained to recognize relevant features from images with different scales. In this case, the differences are even higher.

\subsection{Large Scale ResNet Experiment}

\begin{table}[t]
\caption{Test accuracy on CIFAR-10 using the ResNets architecture. The first column corresponds to the original ResNet model. The ResNet-avg/max/conv-1x1 models come with an additional penultimate dense layer of size $4096$, in order to match the number of parameters in \SAN{} \label{tab:resnet_acc}.}
\vspace{0.3cm}
\centering 
\begin{tabular}{L{1.cm}|c|c|c|c|c|c}
\toprule

{\bf } & {\bf original} & {\bf ResNet-avg} & {\bf ResNet-max}  & {\bf ResNet-conv1x1} & {\bf ResNet-\SAN} \\
\midrule
 error & 0.0735 & 0.0724 &  0.0782 & 0.0780 & {\bf 0.0697} \\ 
\bottomrule
\end{tabular}
\end{table}

In previous experiments we deliberately use a rather simple network in order the examine the effect of only alternating the aggregation method. This allows for the assessment of the methods performance in isolation from any additional layers which could further improve models regularization effect and which are necessary to efficiently train a vast network (such as, for instance, batch norm\cite{ioffe2015batchnorm}). In this experiment, we tested the impact of using \SAN{} in a large-scale network. 

For this purpose we chose the ResNet-56 model\cite{he2016resnet}, which consists of $56$ layers. The original ResNet uses the global average pooling approach followed by a dense layer, which projects the data into the output dimension. Our modification relies on replacing global pooling by \SAN{}. As introduction of the \SAN{} with $4096$ vectors comes at a cost of increased number of parameters, we added an additional, penultimate dense layer with hidden dimension $4096$ to the ResNet for the average- and the max-pooling, and the conv-1x1, in order to allow for fair comparison. 

The results for CIFAR-10 dataset are reported in Table \ref{tab:resnet_acc}. It is evident that the introduction of \SAN{} to the original ResNet architecture led to the improvement of classification accuracy. Moreover, \SAN{} outperformed other aggregation approaches.

\section{Conclusion}

In this paper, we proposed a novel aggregation network, \SAN{}, for processing structured data. Our architecture is based on recent methodology used for learning from permutation invariant structures (sets) \cite{zaheer2017deep}. In addition, to Zaheer's work, we showed that for a sufficiently large number of neurons, \SAN{} allows to preserve the whole information contained in the input. This theoretical result was experimentally confirmed applying convolutional network to image data. Conducted experiments demonstrated that the replacing of global pooling by \SAN{} in typical neural networks used for processing images leads to higher performance of the model.

\section*{Acknowledgements}

This work was partially supported by the National Science Centre (Poland) grants numbers: 2018/31/B/ST6/00993, 2017/25/B/ST6/01271 and 2015/19/D/ST6/01472.

%
%
\bibliographystyle{splncs04}
\bibliography{bibliography}

\begin{thebibliography}{10}
\providecommand{\url}[1]{\texttt{#1}}
\providecommand{\urlprefix}{URL }
\providecommand{\doi}[1]{https://doi.org/#1}

\bibitem{brin1998anatomy}
Brin, S., Page, L.: The anatomy of a large-scale hypertextual web search
  engine. Computer networks and ISDN systems  \textbf{30}(1-7),  107--117
  (1998)

\bibitem{ciresan2011flexible}
Ciresan, D.C., Meier, U., Masci, J., Maria~Gambardella, L., Schmidhuber, J.:
  Flexible, high performance convolutional neural networks for image
  classification. In: IJCAI Proceedings-International Joint Conference on
  Artificial Intelligence. vol.~22, p.~1237. Barcelona, Spain (2011)

\bibitem{coley2017convolutional}
Coley, C.W., Barzilay, R., Green, W.H., Jaakkola, T.S., Jensen, K.F.:
  Convolutional embedding of attributed molecular graphs for physical property
  prediction. Journal of chemical information and modeling  \textbf{57}(8),
  1757--1772 (2017)

\bibitem{cramer1936some}
Cram{\'e}r, H., Wold, H.: Some theorems on distribution functions. Journal of
  the London Mathematical Society  \textbf{1}(4),  290--294 (1936)

\bibitem{defferrard2016convolutional}
Defferrard, M., Bresson, X., Vandergheynst, P.: Convolutional neural networks
  on graphs with fast localized spectral filtering. In: Advances in Neural
  Information Processing Systems. pp. 3844--3852 (2016)

\bibitem{frasconi1998general}
Frasconi, P., Gori, M., Sperduti, A.: A general framework for adaptive
  processing of data structures. IEEE transactions on Neural Networks
  \textbf{9}(5),  768--786 (1998)

\bibitem{van2004short}
van Ginkel, M., Hendriks, C.L., van Vliet, L.J.: A short introduction to the
  radon and hough transforms and how they relate to each other. Delft
  University of Technology  (2004)

\bibitem{he2016resnet}
{He}, K., {Zhang}, X., {Ren}, S., {Sun}, J.: Deep residual learning for image
  recognition. In: 2016 IEEE Conference on Computer Vision and Pattern
  Recognition (CVPR). pp. 770--778 (June 2016). \doi{10.1109/CVPR.2016.90}

\bibitem{he2015spatial}
He, K., Zhang, X., Ren, S., Sun, J.: Spatial pyramid pooling in deep
  convolutional networks for visual recognition. IEEE transactions on pattern
  analysis and machine intelligence  \textbf{37}(9),  1904--1916 (2015)

\bibitem{Huang2017DenselyCC}
Huang, G., Liu, Z., Weinberger, K.Q.: Densely connected convolutional networks.
  In: 2017 IEEE Conference on Computer Vision and Pattern Recognition (CVPR).
  pp. 2261--2269 (2017)

\bibitem{iizuka2017globally}
Iizuka, S., Simo-Serra, E., Ishikawa, H.: Globally and locally consistent image
  completion. ACM Transactions on Graphics (TOG)  \textbf{36}(4), ~107 (2017)

\bibitem{ioffe2015batchnorm}
Ioffe, S., Szegedy, C.: Batch normalization: Accelerating deep network training
  by reducing internal covariate shift. In: Proceedings of the 32Nd
  International Conference on International Conference on Machine Learning -
  Volume 37. pp. 448--456. ICML'15, JMLR.org (2015)

\bibitem{karpathy2015deep}
Karpathy, A., Fei-Fei, L.: Deep visual-semantic alignments for generating image
  descriptions. In: Proceedings of the IEEE conference on computer vision and
  pattern recognition. pp. 3128--3137 (2015)

\bibitem{kim2014convolutional}
Kim, Y.: Convolutional neural networks for sentence classification. arXiv
  preprint arXiv:1408.5882  (2014)

\bibitem{kipf2016semi}
Kipf, T.N., Welling, M.: Semi-supervised classification with graph
  convolutional networks. arXiv preprint arXiv:1609.02907  (2016)

\bibitem{krizhevsky2009learning}
Krizhevsky, A., Hinton, G.: Learning multiple layers of features from tiny
  images. Tech. rep., Citeseer (2009)

\bibitem{maas-EtAl:2011:ACL-HLT2011IMDB}
Maas, A.L., Daly, R.E., Pham, P.T., Huang, D., Ng, A.Y., Potts, C.: Learning
  word vectors for sentiment analysis. In: Proceedings of the 49th Annual
  Meeting of the Association for Computational Linguistics: Human Language
  Technologies. pp. 142--150 (2011)

\bibitem{murphy2018janossy}
Murphy, R.L., Srinivasan, B., Rao, V., Ribeiro, B.: Janossy pooling: Learning
  deep permutation-invariant functions for variable-size inputs. In:
  International Conference on Learning Representations (2019),
  \url{https://openreview.net/forum?id=BJluy2RcFm}

\bibitem{Pang+Lee:05aMR}
Pang, B., Lee, L.: Seeing stars: Exploiting class relationships for sentiment
  categorization with respect to rating scales. In: Proceedings of the ACL
  (2005)

\bibitem{Qi2017PointNetDL}
Qi, C.R., Su, H., Mo, K., Guibas, L.J.: Pointnet: Deep learning on point sets
  for 3d classification and segmentation. 2017 IEEE Conference on Computer
  Vision and Pattern Recognition (CVPR) pp. 77--85 (2017)

\bibitem{radon1986determination}
Radon, J.: On the determination of functions from their integral values along
  certain manifolds. IEEE transactions on medical imaging  \textbf{5}(4),
  170--176 (1986)

\bibitem{real2018amoeba}
Real, E., Aggarwal, A., Huang, Y., Le, Q.V.: Regularized evolution for image
  classifier architecture search. arXiv preprint arXiv:1802.01548  (2018)

\bibitem{ronneberger2015u}
Ronneberger, O., Fischer, P., Brox, T.: U-net: Convolutional networks for
  biomedical image segmentation. In: International Conference on Medical image
  computing and computer-assisted intervention. pp. 234--241. Springer (2015)

\bibitem{sannai2019universal}
Sannai, A., Takai, Y., Cordonnier, M.: Universal approximations of permutation
  invariant/equivariant functions by deep neural networks. arXiv preprint
  arXiv:1903.01939  (2019)

\bibitem{vaswani2017attention}
Vaswani, A., Shazeer, N., Parmar, N., Uszkoreit, J., Jones, L., Gomez, A.N.,
  Kaiser, {\L}., Polosukhin, I.: Attention is all you need. In: Advances in
  Neural Information Processing Systems. pp. 5998--6008 (2017)

\bibitem{xiao2017online}
Xiao, H., Rasul, K., Vollgraf, R.: Fashion-mnist: a novel image dataset for
  benchmarking machine learning algorithms. arXiv preprint arXiv:1708.07747
  (2017)

\bibitem{xu2018how}
Xu, K., Hu, W., Leskovec, J., Jegelka, S.: How powerful are graph neural
  networks? In: International Conference on Learning Representations (2019),
  \url{https://openreview.net/forum?id=ryGs6iA5Km}

\bibitem{zaheer2017deep}
Zaheer, M., Kottur, S., Ravanbakhsh, S., Poczos, B., Salakhutdinov, R.R.,
  Smola, A.J.: Deep sets. In: Advances in Neural Information Processing
  Systems. pp. 3391--3401 (2017)

\end{thebibliography}

\newpage
\appendix

\section{Graph processing}

Predicting the properties of chemical compounds is one of basic problems in medical chemistry. Since the laboratory verification of every molecule is very time and resource consuming, the current trend is to analyze its activity with machine learning approaches. Typically, a graph of chemical compound is represented as a fingerprint, which encodes predefined chemical patterns in a binary vector (Figure \ref{fig:finger}). However, one can also apply graph convolutional neural networks to learn from a graph structure directly without any initial transformation. 

\begin{figure}[h]
\centering
\includegraphics[width=0.44\textwidth]{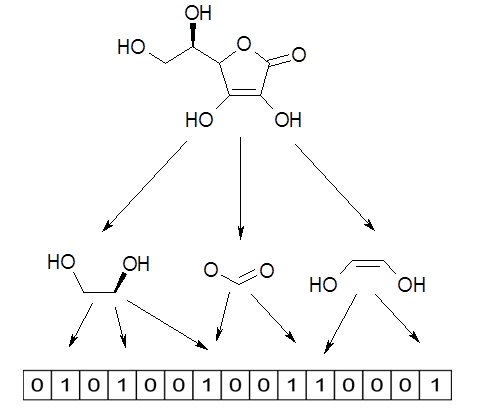}\label{fig:molecule}
\caption{Example of a chemical compound represented by graph and its lossy conversion to fingerprint.\label{fig:finger}} 
\end{figure}

In this experiment, we considered a neural network, {\bf sum-pooling}, developed by Coley et al. \cite{coley2017convolutional}, which is an extension of graph convolutional networks \cite{kipf2016semi}, \cite{defferrard2016convolutional} applied to chemical compounds. Given a set of features extracted by convolutional filters from a graph of compound, the authors use a type of global sum-pooling, which summarizes the results to a fixed length vector. Next, dense layers are applied. We examined the effect of replacing  sum-pooling by a single layer \SAN{}. We used identical number of aggregative neurons in \SAN{} as in the first dense layer of Coley's network to provide equal numbers of parameters in both approaches. Additionally, for a comparison we used a dense network with a comparable number of parameters applied on classical ECFP fingerprint\footnote{\url{https://docs.chemaxon.com/display/docs/Extended+Connectivity+Fingerprint+ECFP}}.

We used Tox21, which is a common benchmark data set for comparing machine learning methods on chemical data. It comprises 12 060 training samples and 647 test samples. For each sample, there are 12 binary labels that represent the outcome (active/inactive) of 12 different toxicological experiments. In consequence, we evaluate methods on 12 classification tasks.

The results presented in Table \ref{tab:chem} indicate that replacing pooling layer by \SAN{} increased the performance in most cases. Although the difference between these methods is not high, obtained model allows to better predict the activity of compounds. Moreover, it is evident that the use of convolutional network significantly outperformed fingerprint-based approach. 

\begin{table}[t]
\caption{ROC AUC scores for Tox-21\label{tab:chem}.}
\centering
\begin{tabular}{lccc}
\toprule
{} &    {\bf ECFP} &  {\bf sum-pooling}  &  {\bf SAN} \\
\midrule
ahr        &  $0.783402$ &  $0.843292$ &  $\textbf{0.849239}$ \\
ar         &  $\textbf{0.678241}$ &  $0.464120$ &  $0.589120$ \\
ar-lbd     &  $0.610887$ &  $0.735887$ &  $\textbf{0.782258}$ \\
are        &  $0.691032$ &  $0.804965$ &  $\textbf{0.837223}$ \\
aromatase  &  $0.704507$ &  $\textbf{0.775510}$ &  $0.710601$ \\
atad5      &  $0.595447$ &  $\textbf{0.612683}$ &  $0.568130$ \\
er         &  $0.591733$ &  $0.606814$ &  $\textbf{0.621660}$ \\
er-lbd     &  $0.467754$ &  $0.723913$ &  $\textbf{0.736957}$ \\
hse        &  $0.678906$ &  $\textbf{0.734375}$ &  $0.688672$ \\
mmp        &  $0.776911$ &  $0.864057$ &  $\textbf{0.866305}$ \\
p53        &  $0.502083$ &  $0.728720$ &  $\textbf{0.750149}$ \\
ppar-gamma &  $0.448340$ &  $\textbf{0.700133}$ &  $0.619389$ \\
\bottomrule
\end{tabular}
\end{table}

\section{Sentiment Classification}

Text documents are often represented as sequences of word embeddings and processed by convolutional or recurrent networks, followed by a pooling layer\cite{kim2014convolutional}. In this experiment, the use of \SAN{} is compared to the average- and max- pooling methods in the problem of sentiment classification. For this purpose we consider the IMDB Movie Review and the Movie Review (MR) dataset from~\cite{Pang+Lee:05aMR}. The IMDB~\cite{maas-EtAl:2011:ACL-HLT2011IMDB} consists of $50000$ samples, from which $20\%$ is used for test dataset. The MR review has $10662$ documents. $10\%$ of the data is used both for validation and test examples. 

The embedding layer size used in this experiment is set to $128$ for the IMDB and $64$ for MR.  The vocabulary used for the one-hot encoding of the inputs to the embedding includes $10000$ most frequent words from each dataset. The so obtained word representations are then passed to a single LSTM recurrent network, which produces outputs for each word in the sequence. 
Those outputs are aggregated to form a single vector and passed to the predictor head network. In this setting, we test the \SAN{} approach against the avg- and max- pooling methods. The number of tested output neurons is from the set $\{256,512\}$ for IMDB and $\{128, 256\}$ for MR. Again, to allow for the same number of parameters in the avg- and max- models, we add and additional fully connected layer after the pooling, with the hidden neurons size matching the size of output neurons in \SAN{}. The networks are trained for $5$ epochs with the \textit{Adam} optimization algorithm. The learning rate is equal to $1e-2$ and batch size is $128$ . The results are presented in Table~\ref{tab:nlp}.

\begin{table}[t]
\caption{Test accuracy for the IMDB and MR datasets. The MR model uses less parameters, as it also has significantly less training examples.
\label{tab:nlp}}.
\vspace{0.5cm}
\centering 
\begin{tabular}{l|c|ccc|}
\toprule

{\bf dataset } & {\bf sizes} & {\bf avg-pooling} & {\bf max-pooling}  & {\bf \SAN} \\
\midrule
IMDB & 	256 & \textbf{0.9016} & 0.8930 & 0.8924 \\ 
     &  512 & 0.8892 & \textbf{0.9044} & 	0.8890 \\ 
\midrule
MR & 	128 & 0.5186 & 	\textbf{0.7656} & 0.7510 \\ 
     &  256 & 0.7480 & \textbf{0.7656} & 0.7471 \\ 
\bottomrule
\end{tabular}
\end{table}

In this experiment the \SAN{} approach gave results comparable to other methods, however was not able to exceed them. The average pooling algorithm also performed rather poorly, behaving nearly as simple random guessing for the most simple MR network (size $128$). The highest accuracy was almost always achieved by the max pooling. This may be the result of identifying the largest activation associated with a sentiment-significant word. In the other pooling approaches, this value would be mixed with other obtained vectors indices, and possibly harder to retrieve. Moreover, increasing the number of parameters in \SAN{} does not result in better outcomes, which may suggest that the model has overfit to the training data. One should also point out that no spacial position information has been added to the representations processed by \SAN{}. This may be a drawback in this specific task, as NLP data are sequence-dependant. 

\end{document}